\documentclass[twoside]{article}

\usepackage[accepted]{aistats2019}

\usepackage[round,comma]{natbib}

\usepackage[utf8]{inputenc} 
\usepackage[T1]{fontenc}    
\usepackage{hyperref}       
\usepackage{url}            
\usepackage{booktabs}       
\usepackage{amsfonts}       
\usepackage{nicefrac}       
\usepackage{microtype}      

\usepackage{amsmath,amsthm,amssymb,amsfonts}
\usepackage{times}

\usepackage{subfig}
\usepackage{float}
\usepackage{algorithmic}
\usepackage[linesnumbered,ruled]{algorithm2e}
\SetKwRepeat{Do}{do}{while}%
\SetAlgoNoEnd%
\usepackage{amsmath}
\usepackage{amsfonts}
\usepackage{graphicx}
\usepackage{amsthm}

\usepackage{abbrevs}
\usepackage{amssymb}

\usepackage[colorinlistoftodos, disable]{todonotes}
\newcommand{\todoy}[2][]{\todo[size=\scriptsize,color=red!20!white,#1]{Yasin: #2}}
\newcommand{\todot}[2][]{\todo[size=\scriptsize,color=blue!20!white,#1]{Tan: #2}}
\newcommand{\todoa}[2][]{\todo[size=\scriptsize,color=cyan!20!white,#1]{Anup: #2}}

\newtheorem{theorem}{Theorem}
\newtheorem{corollary}{Corollary}
\newtheorem{lemma}{Lemma}

\def\G {\mathcal{G}}
\def\cE {\mathcal{E}}
\def\V {\mathcal{V}}
\def\H {\mathcal{H}}
\def\C {\mathcal{C}}
\def\cP {\mathcal{P}}
\def\L {\mathcal{L}}
\def\N {\mathcal{N}}
\def\R {\mathbb{R}}
\def\P {\mathbf{P}}
\newcommand{\E}{{\mathbf E}} 

\newcommand{\Real}{\mathbb R}                        

\renewcommand{\epsilon}{\varepsilon}

\newcounter{assumption}
\newcounter{definition}

\newenvironment{definition}[1][]{\begin{trivlist}\item[] \refstepcounter{definition}%
 {\bf Definition\ \thedefinition\ {\em (#1)} } }{
 \ifvmode\smallskip\fi\end{trivlist}}
 
\newcommand{\argmin}{\mathop{\rm argmin}}
\newcommand{\argmax}{\mathop{\rm argmax}}

\allowdisplaybreaks

\begin{document}

\runningtitle{Sample Efficient Graph-Based Optimization with Noisy Observations}

\twocolumn[

\aistatstitle{Sample Efficient Graph-Based Optimization with Noisy Observations}

\aistatsauthor{Tan Nguyen \And Ali Shameli \And Yasin Abbasi-Yadkori}
\aistatsaddress{Queensland University of Technology \And Stanford University \And Adobe Research}
\aistatsauthor{Anup Rao \And Branislav Kveton}
\aistatsaddress{Adobe Research \And Google Research}
]

\begin{abstract}
We study sample complexity of optimizing ``hill-climbing friendly'' functions defined on a graph under noisy observations. We define a notion of convexity, and we show that a variant of best-arm identification can find a near-optimal solution after a small number of queries that is independent of the size of the graph. For functions that have local minima and are nearly convex, we show a sample complexity for the classical simulated annealing under noisy observations. We show effectiveness of the greedy algorithm with restarts and the simulated annealing on problems of graph-based nearest neighbor classification as well as a web document re-ranking application.  
\end{abstract}

\emph{Note: The first version of this paper appeared in AISTATS 2019. Thank to community feedback, some typos and a minor issue have been identified.
These are fixed in this updated version.}

\section{Introduction}

Stochastic optimization of a function defined on a large finite set frequently arises in many practical problems. 
Instances of this problem include finding the most attractive design for a webpage, the node with maximum influence in a social network, etc. 
There are a number of approaches to this problem. At one extreme, we can use global optimization methods such as simulated annealing, genetic algorithms, or cross-entropy methods~\citep{Rubinstein-1997,Christian-Casella-1999}. Although limited theoretical performance guarantees are available, these methods perform well in a number of practical applications. In the other extreme, we can use the best-arm identification algorithms from bandit literature or hypothesis testing from statistics community~\citep{Mannor-Tsitsiklis-2004,Even-Dar-2006,Audi10}. We have stronger sample complexity results for this family of algorithms. The sample complexity states that, with high probability, the algorithm will return a near-optimal solution after a small number of observations, and this number typically grows polynomially with the size of the decision set. These methods however often perform poorly when the size of the decision set is large. In such problems, it is important to exploit the structure of the specific problems to speed up the optimization. 
If appropriate features are available and it can be assumed that the objective function is linear in these features, then algorithms designed for bandit linear optimization are applicable~\citep{Auer-2002}. More generally, kernel-based or neural network based bandit algorithms are available for problems with non-linear objective functions~\citep{srinivas-2010}.  


We are particularly interested in problems where the item similarity is captured by a graph structure. 
In general, and with no further conditions, we cannot hope to show non-trivial sample complexity rates. We observe that in many real-world applications, the objective function is easy and \emph{hill-climbing friendly}, in the sense that from many nodes, there exist a monotonic path to the global minima. We make this property explicit by defining a notion of \emph{convexity} for functions defined on graphs. 
Under this condition, we show that a hill-climbing procedure that uses a variant of best-arm identification as a sub-routine has a small sample complexity that is independent of the size of graph. In the presence of local minima, this greedy procedure might require many restarts, and might not be efficient in practice. 
Simulated annealing is commonly used in practice for such problems. 
We also define a notion of nearly convex functions that allows for existence of shallow local minima. We show that for nearly convex functions and using an appropriate estimation of function values, the classical simulated annealing procedure finds near optimal nodes after a small number of function evaluations. 

While asymptotic convergence of simulated annealing is studied extensively in the literature, there are only few finite-time convergence rates for this important algorithm. \cite{Pydi18} show finite-time convergence rates for the algorithm, but they consider only deterministic functions and their rates scale with the size of the graph. These results, that are obtained under very general conditions, do not quite explain success of the simulated annealing algorithm in large-scale problems. In practice, simulated annealing finds a near-optimal solution with a much smaller number of function evaluations. Our bounds are in terms of the convexity of the function, and the rates do not scale with the size of the graph. Additionally, our results hold in the noisy setting. \cite{Bout17} show convergence rates for simulated annealing applied to noisy functions. Their results, however, appear to have several gaps. \todoy{explain more or remove this!} 


\subsection{Related Work}


A best-arm identification algorithm from bandit literature can find a near-optimal node, and the time and sample complexity is linear in the number of nodes~\citep{Mannor-Tsitsiklis-2004,Even-Dar-2006,Audi10,Jamieson-Nowak-2014,Kaufmann-Cappe-Garivier-2016}. Such complexity is not acceptable when the size of the graph is very large. We are interested in designing and analyzing algorithms that find a near-optimal node, and the sample and computational complexity is sublinear in the number of nodes. 

Bandit algorithms for graph functions are studied in \cite{Valkoetal2014}. The sample complexity of these algorithms scale with the smoothness of the function, but the computational complexity scales linearly with the number of nodes. We are interested in problems where the graph is very large, and we might have access to only local information. As we will see in the experiments, the \textsc{SpectralBandit} algorithm of \cite{Valkoetal2014} is not applicable to problems with large graphs. Further, we study a different notion of regularity that is inspired by convexity in continuous optimization. Bandit problems for large action sets are also  studied under a number of different notions of regularity~\citep{Bubecketal-2009,APS-2011,Carpentier-Valko-2015,Grill-Valko-Munos-2015}.
 
\subsection{Contributions}
 
We cannot hope to achieve sublinear rate without further conditions on the function. We define a notion of strong convexity for functions defined on a graph. For strongly convex functions, we design an algorithm, called \textsc{Explore-Descend}, to find the optimal node with guaranteed error probability. The \textsc{Explore-Descend} algorithm uses a \emph{best-arm identification} algorithm as a submodule. The best-arm identification problem is the problem of finding the optimal action given a sampling budget. For \emph{nearly convex functions}, we show that the classical simulated annealing algorithm finds the global minima in a reasonable time. For convex functions, the \textsc{Explore-Descend} has lower sample complexity, but simulated annealing can handle non-convex functions. 

We study the empirical performance of \textsc{Explore-Descend} and simulated annealing in two applications. First, we consider the problem of content optimization in a digital marketing application. In each round, a user visits a webpage that is composed of a number of components. The learning agent chooses the contents of these components and shows the resulting page to the user. The user feedback, in the form of click/no click or conversion/no conversion, is recorded and used to update the decision making policy. The objective is to return a page configuration with near-optimal  click-through rate. We would like to find such a configuration as quickly as possible with a small  number of user interactions. In this problem, each page configuration is a node of the graph, and two nodes are connected if the corresponding page configurations differ in only one component. 

Our second application is the problem of nearest neighbor search and classification. Given a set of points $\V=\{x_1, x_2, \ldots, x_n\}\subset \Real^d$, a query point $y\in\Real^d$, and a distance function $\ell:\Real^d\times \Real^d \rightarrow \Real$, we are interested in solving $\argmin_{x\in \V} \ell(x,y)$. 
A trivial solution to this problem is to examine all points in $\V$ and return the point with the smallest distance value. The computational complexity of this method is $O(n)$, and is not practical in large-scale problems. An approximate nearest neighbor method returns an approximate solution in a sublinear time. A class of approximate nearest neighbor methods that is particularly suitable for big-data applications is the graph-based nearest neighbor search~\cite{Arya-Mount-1993}. 
In a graph-based search, we construct a proximity graph in the training phase, and perform a hill-climbing search in the test phase. 
To improve the performance, we perform an additional smoothing procedure that replaces the value of a node by the average function values in a vicinity of the node. In practice, these average values are estimated by performing random walks, and hence the problem is a graph optimization problem with noisy observations. We show that the proposed graph optimization technique outperforms popular nearest neighbor search methods such as KDTree and LSH in two image classification problems. Interestingly, and compared to these methods, the computational complexity of the proposed technique appears to be less sensitive to the size of the training data. This property is particularly appealing in big data applications.


\subsection{Notations}
  
We use $[K]$ to denote the set $\{1,2,\dots, K\}$. Let $\G=(\V,\cE)$ be a graph with $n$ nodes $\V = \{ 1, \dots, n\}$ and the set of edges $\cE$. Let $f : \V \rightarrow [0,1]$ be some unknown function defined on $\V$. Let $x^* = \argmin_{x \in \V} f(x)$ be the global minimizer of $f$. The goal is to find a node $x$ with small loss $f(x) - f(x^*)$. In this paper, we study the problem where the evaluation of $f$ is noisy, such that we can only observe $g(x; \eta) = f(x) + \eta$, where $\eta$ is a zero-mean $R$-sub-Gaussian random variable, meaning that for any $\lambda\in\R$, $\E(e^{\lambda \eta}) \le e^{\lambda^2 R^2/2}$.
  
Let $\N_x$ denote the set of neighbors of node $x\in \V$. Let $\overline \N_x = \N_x\cup\{x\}$. For simplicity, we assume all nodes have the same degree and we let $d=|\N_x|$ denote the number of neighbors. We sometimes write $f_x$ to denote the function value $f(x)$. For two nodes $x,y\in \V$, we use $\cP_{x,y}$ to denote all paths from $x$ to $y$ in the graph. We use $\cP_x$ to denote all paths starting from node $x$. We use $\ell(p)$ to denote the length of a path.

\subsection{Convexity}

The general discrete optimization problem is hard for an unrestricted function $f$ and graph $\G$. As such, we study a restricted class of problems that allow for efficient algorithms. Let $\Delta_{x,z} = f_x - f_{z}$ be the amount of improvement if we move from node $x$ to the neighbor $z\in \N_x$. We say path $p = \{x_0 \rightarrow x_1 \rightarrow \dots \rightarrow x_k \}$ is $m$-strongly convex if $\Delta_{x_{i-1},x_i} - \Delta_{x_{i},x_{i+1}} \geq m \Delta_{x_{i},x_{i+1}} > 0$, for all $i=1, \dots, k$. Sometimes we use $\Delta_{x_i}$ to denote $\Delta_{x_i,x_{i+1}}$ if the $m$-strongly convex path is clear from the context. We use $\overline \Delta_{x}$ to denote $\max_{z\in\N_x}\Delta_{x,z}$.
\todot{$\bar{\Delta}_x$ can be negative at local max?}

\begin{definition}[Convex Functions]
\label{defn:convex}
A function $f : \V \rightarrow \R$ defined on a graph $\G=(\V,\cE)$ is (strongly) convex if from any node $x\in \V$, there exists a (strongly) convex path to the global minima $x^*$. 
\end{definition}
For a nearly convex function, as defined next, the convexity condition is not satisfied at all points. 
\begin{definition}[Nearly Convex Functions]
\label{defn:nearlyconvex}
Let $\alpha>0$ be a parameter. Let $\C$ be the set of points such that $\overline \Delta_{x} \ge \alpha (f_x - f_{x^*})$. We say function $f : \V \rightarrow \R$ is $(\alpha,c,r)$-nearly convex if for any $x\in \V\setminus \C$, there exists a $y\in \C$ and a $p\in \cP_{xy}$ such that $\ell(p)\le r$ and $\max_{z\in p} f(z) - f(x) \le c$. 
\end{definition}
A path that satisfies the above conditions is called a low energy path. A convex function is also a $(\frac{m}{m+1},0,0)$-nearly convex function. 
\begin{lemma}
\label{lem:stopping-alg}
We have that $f_x - f_{x^*} \le \frac{m+1}{m} \Delta_x$. \todoa{If $x,x^*$ are next to each other, then we should have $f_x - f_{x^*} = \Delta_x$ ?}
\todoy{yes, which is fine, right?}
\end{lemma}
\begin{proof}
Let $p = \{y=x \rightarrow y_1 \rightarrow \dots \rightarrow y_k \}$ be a strongly convex path. Then 
\[
\Delta_{y_k} \le \frac{1}{1+m} \Delta_{y_{k-1}} \le \dots \le \frac{1}{(1+m)^k} \Delta_x \;.
\]
We have that
\begin{align*}
f_x - f_{x_k} &= \sum_{i=1}^k (f_{y_{i-1}} - f_{y_i}) = \sum_{i=1}^k \Delta_{y_{i-1}}\\ 
&= \Delta_x + \sum_{i=1}^{k-1} \Delta_{y_i} \le \sum_{i=0}^{k-1} \frac{\Delta_x}{(1+m)^i} \;.
\end{align*}
We conclude that $f_x - f_{y_k} \le \frac{m+1}{m} \Delta_x$, from which the statement follows. 
\end{proof}
Concave and nearly concave functions are similarly defined. The convexity condition allows for efficient algorithms. The intuition behind our analysis is that, by strong convexity, if node $x$ is far from the global minima the improvement $\Delta_x$ is large. When degree $d$ is relatively small, the local search methods have a sufficiently large probability of hitting a good direction (the strongly convex path). So either we are already close to the global minima, or we are far from the global minima and with a constant probability we make a large improvement.

\section{Approximating the Global Minimizer}
\label{approximation}

We analyze two algorithms for the graph-based optimization. The first algorithm is a local search procedure where in each round, the learner attempts to identify a good neighbor and move there. This algorithm is analyzed under the strong convexity condition. The second algorithm is the well-known simulated annealing procedure with an exponential transition function. 
We provide high probability error bounds for the greedy method, while the sample complexity of the more complicated simulated annealing is analyzed in expectation. 


\subsection{The Local Search Algorithm and Best-Arm Identification}
\label{sec:hillclimbing}

The greedy approach is an intuitive approach to the graph optimization problem: We start from a random node, and at each node and given a sampling budget, we explore its neighbors to find the best neighbor. The problem that we solve in each node can be viewed as a fixed-budget best-arm identification problem~\cite{Audi10}. In a fixed-budget setting, the learner plays actions for a number of rounds in some fashion and returns an approximately optimal arm at the end. An example of an algorithm designed for the fixed-budget setting is the \textsc{SuccessiveReject} algorithm of \cite{Audi10}. \todoy{explain why fixed budget algorithms are more appropriate} 

Before describing the best-arm identification algorithm, we introduce some notation. Let $K$ be an integer, $[K]$ be the set of arms in a bandit problem, $\mu_i$ be the mean value of arm $i$, and $\widehat \mu_{i,t}$ be the empirical mean of arm $i$ after $t$ observations. Let $i^*=\argmax_{i\in [K]} \mu_i$ be the optimal arm (break tie arbitrarily) and $\mu_* = \mu_{i^*}$ be its mean, and let $\Delta_i$ be the \textit{optimality gap} of arm $i \neq i^*$, i.e. $\Delta_i = \mu_* - \mu_i$. Without loss of generality assume that $\Delta_1 \le \Delta_2 \le \dots \le \Delta_K$. Let $B$ be the budget. Define
\begin{align*}
H &= \max_{i\in [K]} i \Delta_{i}^{-2}\,,\\  
\overline{\log}(K) &= \frac{1}{2} + \sum_{i=2}^K \frac{1}{i}\,,\\ 
B_k &= \left\lceil \frac{1}{\overline{\log}(K)} \frac{B-K}{K+1-k} \right\rceil\,, B_0=0 \;.
\end{align*}
The \textsc{SuccessiveReject} algorithm is a procedure with $K-1$ rounds where in each round one arm is eliminated. Let $A_k$ be the remaining arms at the beginning of round $k\in [K]$. In round $k$, each arm $i\in A_k$ is selected for $B_k - B_{k-1}$ rounds. At the end of this round, the arm with the smallest empirical mean is removed. 

\begin{theorem}[\cite{Audi10}]
\label{thm:SR}
The probability of error of the \textsc{SuccessiveReject} algorithm with a budget of $B$ is smaller than
\[
e_B = \frac{K(K-1)}{2} \exp\left(-\frac{B-K}{\overline{\log}(K) H} \right) \;.
\]
\end{theorem}


In the \textsc{Explore-Descend} procedure, we use the above bandit algorithm to find the best neighbor in each round. More specifically, for a node $x\in\V$, we solve a best-arm identification problem with action values $\mu_i = f(x) - f(z_i)$, where $z_i$ is the $i$th neighbor of node $x$. 
\todoa{Hill 'climbing', descent direction and the good arm definitions are a bit confusing as sometimes they speak of ascent and sometimes descent} We then move to the chosen neighbor and repeat the process until the budget is consumed. We call this approach ``Explore and Descend'' and it is detailed in Algorithm~\ref{alg_explore_descend}. 

\SetAlgoNoEnd
\begin{algorithm}

    \caption{Explore and Descend}
    \label{alg_explore_descend}

    \SetKwInOut{Input}{Input}
    \SetKwInOut{Output}{Output}
    \SetKwFunction{DescentOracle}{DescentOracle}
    
	\Input{graph $\G$, starting node $x_0$, budget $T$}
    \Output{$x$ such that $x = x^*$ with high probability}
    Per round budget $\{T_1,\dots,T_S\}$ such that $\sum_{s=1}^S T_s= T$\;
    \For{$s=0,\dots,S-1$}{
        $x_{s+1} \leftarrow$ \DescentOracle{$\G,x_{s},T_s$}\;
    }
    \Return $x_S$
\end{algorithm}

The algorithm depends on the subroutine $\DescentOracle(\G,x_{s-1},T_s)$. This subroutine is an implementation of the best-arm bandit algorithms described earlier with the decision set $\overline\N_x$ and budget $T_s$. 
\begin{corollary}
\label{cor:bound}
Let $f$ be a $m$-strongly convex function, let $x_1 \in \V$ be an arbitrary starting node, and let $\{T_1,\dots,T_S\}$, $\sum_{s=1}^S T_s = T$, be per round budgets. Assume that $S$ is sufficiently large for the steepest descent path from $x_1$ to reach $x^*$. Let $e_T = \P(x_{S+1} \neq x^*)$ be the probability of error of the \textsc{Explore-Descend} algorithm. Then $e_T$ is upper bounded by the following inequality:
\begin{equation}
\label{eq:ED0}
e_T \leq \frac{d(d-1)}{2} \sum_{s=1}^{S} \exp\left(-\frac{(T_s-d) \Delta^2_{1,s}}{d \overline{\log}(d)} \right) \,,
\end{equation}
where $\Delta_{i,s}$ is the $i$-th optimality gap at node $x_s$.
\end{corollary}
\begin{proof}
Let $H_{s} = \max_{i\in [d]} i \Delta_{i,s}^{-2}\,$. 
In round $s$, the \textsc{Explore-Descend} algorithm uses \textsc{SuccessiveReject} for the subroutine $\DescentOracle(f,\G,x_s,T_s)$. Thus, by Theorem \ref{thm:SR}, the probability of error in round $s$ is upper bounded by
\[ e_{T_s} \leq \frac{d(d-1)}{2}  \exp\left(-\frac{T_s-d}{H_{s} \overline{\log}(d)} \right). \]
Using the union bound on the above inequality for round $s=1, \dots, S$, and the loose upper bound $H_s \leq d \Delta^{-2}_{1,s}\,$, we obtain Equation \ref{eq:ED0}.
\end{proof}

On the other hand, we can also apply the \textsc{SuccessiveReject} algorithm directly on $V$, the set of all nodes of the graph. In this case, each node in the graph is one arm, and the graph structure is disregarded. The probability of error of \textsc{SuccessiveReject}, using Theorem \ref{thm:SR} with $K=n$ and $B=T$, is upper bounded by the following inequality:
\begin{equation}
\label{eq:SR1}
e'_T \le \frac{n(n-1)}{2} \exp\left(-\frac{T-n}{H \overline{\log}(n)} \right) \;.
\end{equation}
Using the loose upper bound $H \le n \Delta_1^{-2}$, the above can be written as:
\begin{equation}
\label{eq:SR2}
e'_T \le \frac{n(n-1)}{2} \exp\left(-\frac{(T-n)\Delta_1^2}{n \overline{\log}(n)} \right) \;.
\end{equation}

Note that the error bounds of \textsc{SuccessiveReject} given in Equations \ref{eq:SR1} and \ref{eq:SR2} is vacuous when $T \le n$. In contrast, Equation \ref{eq:ED0} provides meaningful error bounds for \textsc{Explore-Descend}, even in this small budget regime, when $T \leq n$. Additionally, we see that the error bound for \textsc{Explore-Descend} are independent of the size of the graph. Rather, it depends on $S$, which in turn depends on the convexity constant $m$. Larger $m$ means the function is steeper and fewer steps (smaller $S$) are required to reach the global optimum.

\subsection{Nearly Convex Problems: Simulated Annealing}

Given that we have access only to noisy observations, we  consider the following Metropolis-Hastings Algorithm with exponential weights:
\[ 
P(x \rightarrow y) =
  \begin{cases}
    \frac{1}{d} \min\left(1, e^{\gamma (\widehat f_{x,s} - \widehat f_{y,s})} \right)        & \quad \text{if } y \in \N_x \\
    1 - \sum_{z\in \N_x} P(x\rightarrow z)  & \quad \text{if } y=x 
  \end{cases}
\]  
To simplify the analysis, we use a fixed time-independent inverse temperature, although in practice a time-varying inverse temperature might be preferable. We estimate each function value by $s=2r\gamma^2 R^2$ samples. Next, we provide sample complexity bounds for the above procedure in expectation. 
\begin{theorem}
For a $m$-strongly convex function $f$, let $\alpha = \frac{m}{m+1}$ and $\{x_0 \rightarrow x_1 \rightarrow x_2 \rightarrow \dots \}$ be the path generated by \textsc{SimAnnealing}
from an arbitrary initial node $x_0$. 
Given a constant $\epsilon$ and with the choice of $\gamma = \frac{d}{e \alpha \epsilon}$, after $t \geq \frac{\log\left(\alpha(f_{x^*}-f_{x_0})/(\epsilon d)\right)}{\log\left( d/(d - \alpha) \right)}$ rounds, we have 
\[
\E[f(x_{t+1}) - f(x^*)] \le \epsilon \;. 
\]
For a nearly convex function, let $\beta = 1 -\frac{\alpha e^{-c r \gamma}}{d^{r+1}}$ and let $F=\max_{y \in \V} f(y) - f(x^*)$. With the choice of $\gamma = 1/c$ and after $t \ge \frac{r}{\log(1/\beta)} \log \left( F\alpha\gamma\right)$ rounds, we have that 
\[
\E(f(x_{t+r+1}) - f(x^*)) \le \frac{3}{\alpha \gamma} d^{r+1} e^{r} \;.
\]
\end{theorem}
\begin{proof}
Let $y(x_t)$ be the closest point in $\C$ on a low energy path from $x_t$ and let $r(x_t)=r_t$ be the distance to this point from $x_t$. Let $\cP_{x_t}'$ be the set of paths of length less than $r_t+1$ and starting at $x_t$ such that at least one node on the path is not the same as $x_t$. Consider the low energy sub-path starting at $x_t$: 
\[
p=\{z_1=x_t \rightarrow z_2 \rightarrow \dots \rightarrow z_{r(x_t)}=y(x_t) \} \;. 
\]
Let $z(p)$ be the terminating node in a path $p$. Given that function $f$ is $(\alpha,c,r)$-nearly convex, and given that noise is $R$-sub-Gaussian\todoa{No mention about the noise in the theorem statement}, probability that this low-energy path is taken by the algorithm is
\begin{align*}
\P(p) &\le \frac{1}{d^{r_t+1}} \E \left(e^{-\sum_{k=1}^{r_t} \gamma (\widehat f(z_k) - \widehat f(z_{k-1}))} \right) \\
&= \frac{1}{d^{r_t+1}} \E \Bigg(\exp\Big(-\sum_{k=1}^{r_t} \gamma (\widehat f(z_k) -  f(z_{k}))\\ 
&\qquad\qquad\qquad+ \sum_{k=1}^{r_t} \gamma (\widehat f(z_{k-1}) -  f(z_{k-1}))\\ 
&\qquad\qquad\qquad+ \sum_{k=1}^{r_t} \gamma (f(z_{k-1}) -  f(z_{k}))\Big) \Bigg) \\
&\le \frac{1}{d^{r_t+1}} e^{-\gamma (f(z(p)) - f(z_1)) + 2 r_t \gamma^2 R^2/s}\;. 
\end{align*}
Because $p$ is a low energy path, we have $\P(p) \ge \frac{1}{d^{r_t+1}} e^{-c r_t \gamma}$. Let $c_t=f(y(x_t))-f(x_t)$, and $P'=\sum_{p'\in \cP_x'} \P(p')$. Notice that given $x_t$, $c_t$ and $y(x_t)$ are deterministic. We write
\begin{align*}
&\E(f(x_{t+r_t+1})|x_t) \le \P(p) \left( f(x_t) + c_t - \Delta_{y(x_t)} \right)\\ 
&\qquad+ \sum_{p'\in \cP_x'} \P(p') f(z(p'))+ \left(1 - \P(p) - P' \right) f(x_t) \;.
\end{align*} 
The first term on the right side is related to the event that the state follows the low energy path to the state $y(x_t)$ for $r_t$ rounds, and then goes to the best immediate neighbor at state $y(x_t)$. The second term is related to the event that the state is not the same as $x_t$ after $r_t+1$ rounds. Finally, the last term is related to the event that the state stays in $x_t$ for the next $r_t+1$ rounds. If $\Delta_{y(x_t)} \le c_t$, then by Definition~\ref{defn:nearlyconvex}, we already have
\begin{align*}
f(y(x_t)) - f(x^*) &\le \frac{\Delta_{y(x_t)}}{\alpha}  \le \frac{c}{\alpha} \qquad \Rightarrow \\ 
f(x_t) - f(x^*) &\le (r_t + 1/\alpha)c \;.
\end{align*}
Otherwise if $c_t \le \Delta_{y(x_t)}$, we continue as follows:
\begin{align*}
&\E(f(x_{t+r_t+1})|x_t) \le f(x_t) + \frac{1}{d^{r_t+1}} e^{-c r_t \gamma } \left(c_t - \Delta_{y(x_t)} \right)\\ 
&\qquad+ \sum_{p'\in \cP_x'} \P(p') (f(z(p')) - f(x_t)) \\
&\le f(x_t) + \frac{1}{d^{r_t+1}} e^{-c r_t \gamma} \left(c_t - \Delta_{y(x_t)} \right)\\ 
&\qquad+ \sum_{\substack{p'\in \cP_x',\\ f(z(p')) \ge f(x_t)}} \P(p') (f(z(p')) - f(x_t)) \\
&\le \sum_{\substack{p'\in \cP_x',\\ f_{z(p')} \ge f_{x_t}}} \frac{e^{-\gamma (f(z(p')) -f(x_t)) + 2 r_t\gamma^2 R^2/s}}{d^{r_t+1}} (f_{z(p')} - f_{x_t}) \\ 
&\qquad+f(x_t) + \frac{1}{d^{r_t+1}} e^{-c r_t \gamma} \left(c_t - \Delta_{y(x_t)} \right) \\
&\le f(x_t) +\frac{1}{d^{r_t+1}} e^{-c r_t \gamma} \left(c_t - \Delta_{y(x_t)} \right) + \frac{e^{-1+2 r_t \gamma^2 R^2/s}}{\gamma}\,, \\
\end{align*} 
where the last step follows from inequality $e^{-\gamma b} b \le 1/(\gamma e)$. By Definition~\ref{defn:nearlyconvex}, $\Delta_{y(x_t)} \ge \alpha (f(y(x_t)) - f(x^*))$. Thus,
\begin{align*}
c_t - \Delta_{y(x_t)} &\le f(y(x_t))-f(x_t) - \alpha (f(y(x_t)) - f(x^*))\\
&= - (f(x_t) -f(x^*) ) - f(x^*) + \alpha f(x^*)\\ 
&\qquad+ (1-\alpha) f(y(x_t)) \\
&= - (f(x_t) -f(x^*) )\\ 
&\qquad+ (1-\alpha) ( f(y(x_t)) - f(x^*) ) \\
&\le - \alpha (f(x_t) -f(x^*) ) +  (1-\alpha)c r_t \,,
\end{align*}
where we used $f(y(x_t)) \le f(x_t) + c r_t$ in the last step. Let $k=t+r_t+1$. We have,
\begin{align*}
&\E(f(x_{k})|x_t) - f(x^*) \le \left( 1 - \frac{\alpha e^{-c r_t \gamma}}{d^{r_t+1}} \right) (f(x_t) -f(x^*) )\\ 
&\qquad\qquad+ \frac{c r_t e^{-c r_t \gamma} (1-\alpha)}{d^{r_t+1}} +  \frac{e^{-1+2 r_t \gamma^2 R^2/s}}{\gamma} \\
&\qquad\le \left( 1 - \frac{\alpha e^{-c r \gamma}}{d^{r+1}} \right) (f(x_t) -f(x^*) )\\ 
&\qquad\qquad+ \frac{c}{e(1+\log d)} +  \frac{e^{-1+2 r \gamma^2 R^2/s}}{\gamma} \\
&\qquad\le \beta^{\lfloor \frac{t}{r} \rfloor} (f(x_1) - f(x^*))\\ 
&\qquad\qquad+ \frac{1}{1-\beta} \left( \frac{c}{e(1+\log d)} + \frac{e^{-1+2 r \gamma^2 R^2/s}}{\gamma} \right) \,,
\end{align*}
where the second step holds by $r_t\le r$, $\gamma c = 1$, and $\max_r c r e^{-\gamma c r} d^{-r} \le (c/e)/(1+\log d)$, and $\beta$ is defined in the theorem statement. 
Using the fact that $s=2r\gamma^2 R^2$ and $\gamma c = 1$, and a simple calculation shows that after $t \ge \frac{r}{\log(1/\beta)} \log \left( F \alpha \gamma \right)$ rounds, we have that
\begin{align*}
\E(f(x_{t+r+1}) - f(x^*)) &\le \frac{3 d^{r+1} e^{r}}{\alpha \gamma} \;.
\end{align*}
\todoy{I think this is the best we can hope for. So we have something meaningful only if $c$ and $r$ are sufficiently small.}
If $c=0$, then we get that
\begin{align*}
\E(f(x_{k})|x_t) - f(x^*) &\le \left( 1 - \frac{\alpha}{d^{r+1}} \right) (f(x_t) -f(x^*) ) +  \frac{1}{\gamma} \\
&\le \beta^{\lfloor \frac{t}{r} \rfloor} (f(x_1) - f(x^*)) + \frac{1}{(1-\beta)\gamma} \;.
\end{align*}
After $t \ge \frac{r}{\log(1/\beta)} \log \left( F \alpha \gamma \right)$ rounds, we have that
\begin{align*}
\E(f(x_{t+r+1}) - f(x^*)) &\le \frac{2 d^{r+1}}{\alpha \gamma} \;.
\end{align*}
For a strongly convex function, following a similar argument, we have that
\begin{align*}
\E(f(x_{t+1}) - f(x^*)) &\le \left( 1 - \frac{\alpha}{d} \right) (f(x_t) - f(x^*)) + \frac{1}{\gamma e} \;.
\end{align*}
Thus, given $\epsilon$ and with the choice of $\gamma$ in the theorem statement, after $t \geq \frac{\log\left(\alpha(f_{x^*}-f_{x_0})/(\epsilon d)\right)}{\log\left( d/(d - \alpha) \right)}$ rounds, we have $\E[f(x_{t+1}) - f(x^*)] \le \epsilon$. 
\end{proof}
\todoy{compare this with the complexity of ExploreDescend}
For nearly convex problems, the error bound in this theorem is meaningful only if parameter $r$ is small. In our experiments data, this value is $r=1$ or $r=2$.   

\section{Experiments}

We implemented \textsc{Explore-Descend} and \textsc{Simulated Annealing} algorithms and compare them with \textsc{SpectralBandit} of \cite{Valkoetal2014} and \textsc{SuccessiveReject} of \cite{Audi10}. The \textsc{SuccessiveReject} algorithm works by considering all nodes of the graph as a big multi-arm bandit problem. The \textsc{SpectralBandit} uses the adjacency matrix to calculate the Laplacian and Eigenvectors. Both of these algorithms, therefore, require global information of the graph, while our algorithms assume only local information: from one node one can only access its neighbors.

We note that the \textsc{SpectralBandit} algorithm is originally designed to minimize the cumulative regret, which partly explains the poor performance in a best-arm identification problem. In the fixed budget best-arm identification setting, we run \textsc{SpectralBandit} until the budget is consumed and output the most frequently pulled arm, which is, in our experiments, better than taking the arm with the best empirical mean. The algorithm has a high time complexity because of its reliance on matrix operations on matrices and vectors of dimension $n$.  As \textsc{SpectralBandit} does not scale well with the size of the graph due to its matrix operations, we generated a small synthetic graph in order to evaluate it. 

\subsection{Applications in Graph-Structured Best-Arm Identification: Synthetic Data}
\label{sec:synthetic}

First, we evaluated different algorithms on synthetic graphs, which are generated as follows. Each node of the graph is a point $(x, y)$, where $x,y \in \{-D,\dots, -1, 0, 1, \dots, D\}$ for some $D \in \mathrm{Z}^+$. Each node is connected to all of its eight immediate neighbors on the plain. Additionally, random edges are added in, such that the degree of each node is 15. The mean function value is $f(x,y) = 0.8(1 - \frac{x^2 + y^2}{2D^2}) \in [0, 0.8]$. This mean is unknown to the algorithms, i.e. when the algorithm requests the value of $f(x,y)$, it is returned with a stochastic value: 1 with probability $p=f(x,y)$ and 0 with probability $1-p$. It is easy to see that this graph is concave by Definition~\ref{defn:convex}.

The results of the experiment is presented in Figure \ref{fig:synthetic}. The performance measure is the average  sub-optimality gap, i.e. $f(x^*) - f(\hat{x})$, where $\hat{x}$ is the solution returned by the algorithm. The average is taken over 1000 trials. Overall, our algorithms significantly outperform \textsc{SpectralBandit} and \textsc{SuccessiveReject} both in term of time and sub-optimality gap, especially when the budget is smaller than the graph size, which is our intended setting. For \textsc{Simulated Annealing}, we use a fixed inverse temperature $\gamma = 250$. The result could be further improved by optimizing a schedule for this parameter. Interestingly, the number of pulls for each function evaluation also has significant impact on the performance, which can be seen from the plots for \emph{Sim Annealing 1} (1 pull) and \emph{Sim Annealing 5} (5 pulls) in Figure \ref{fig:synthetic}. As for \textsc{Explore-Descend}, we simply allocate the budget equally for each node in the descending path, with the maximum path length set to $4$ for $D=10$ and $20$ for $D=100$. This algorithm is the fastest and also offers the best performance. Source code is available at \url{https://github.com/tan1889/graph-opt}  

\begin{figure*}
\begin{center}
    \includegraphics[width=\columnwidth]{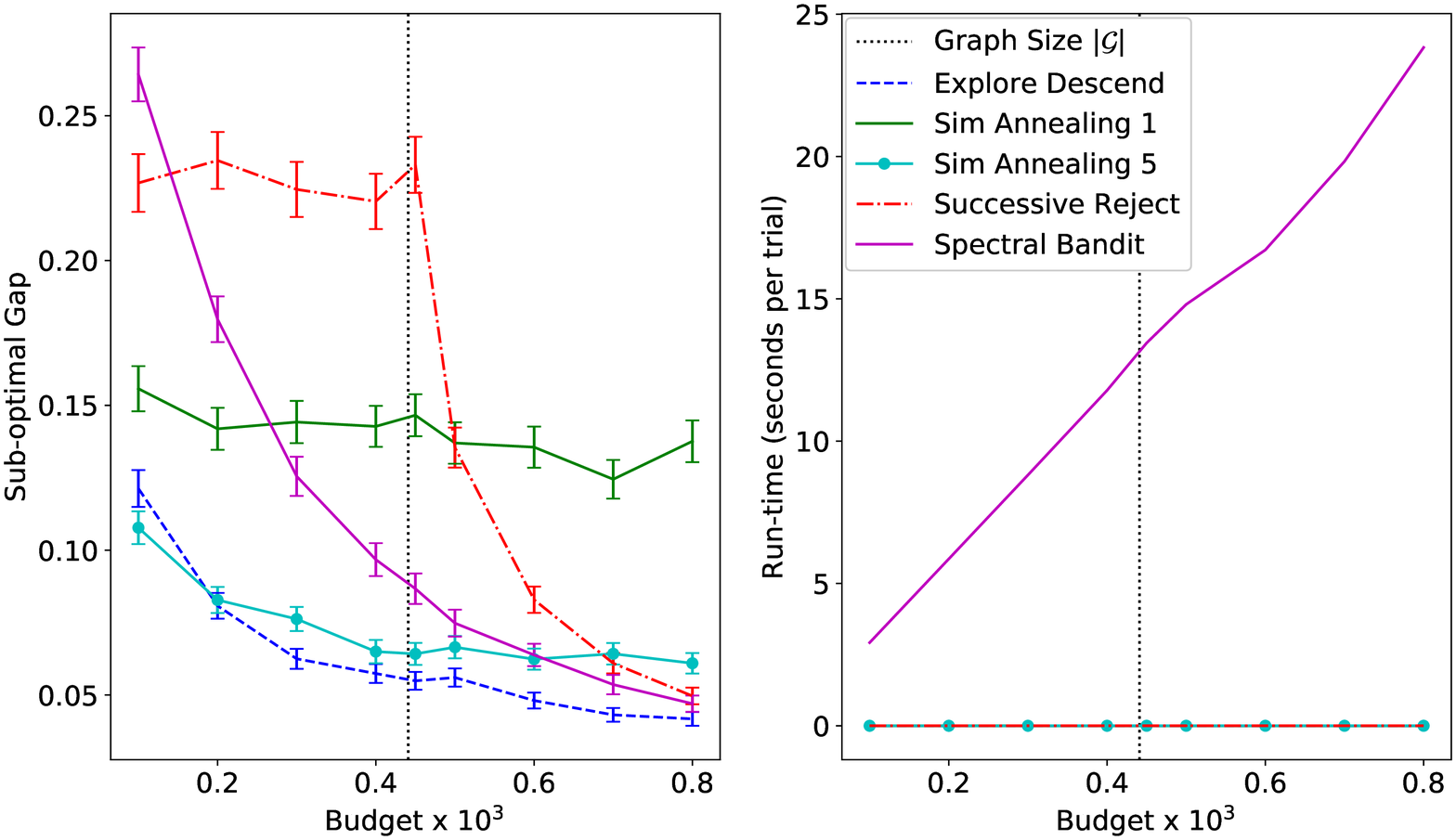}
    \includegraphics[width=\columnwidth]{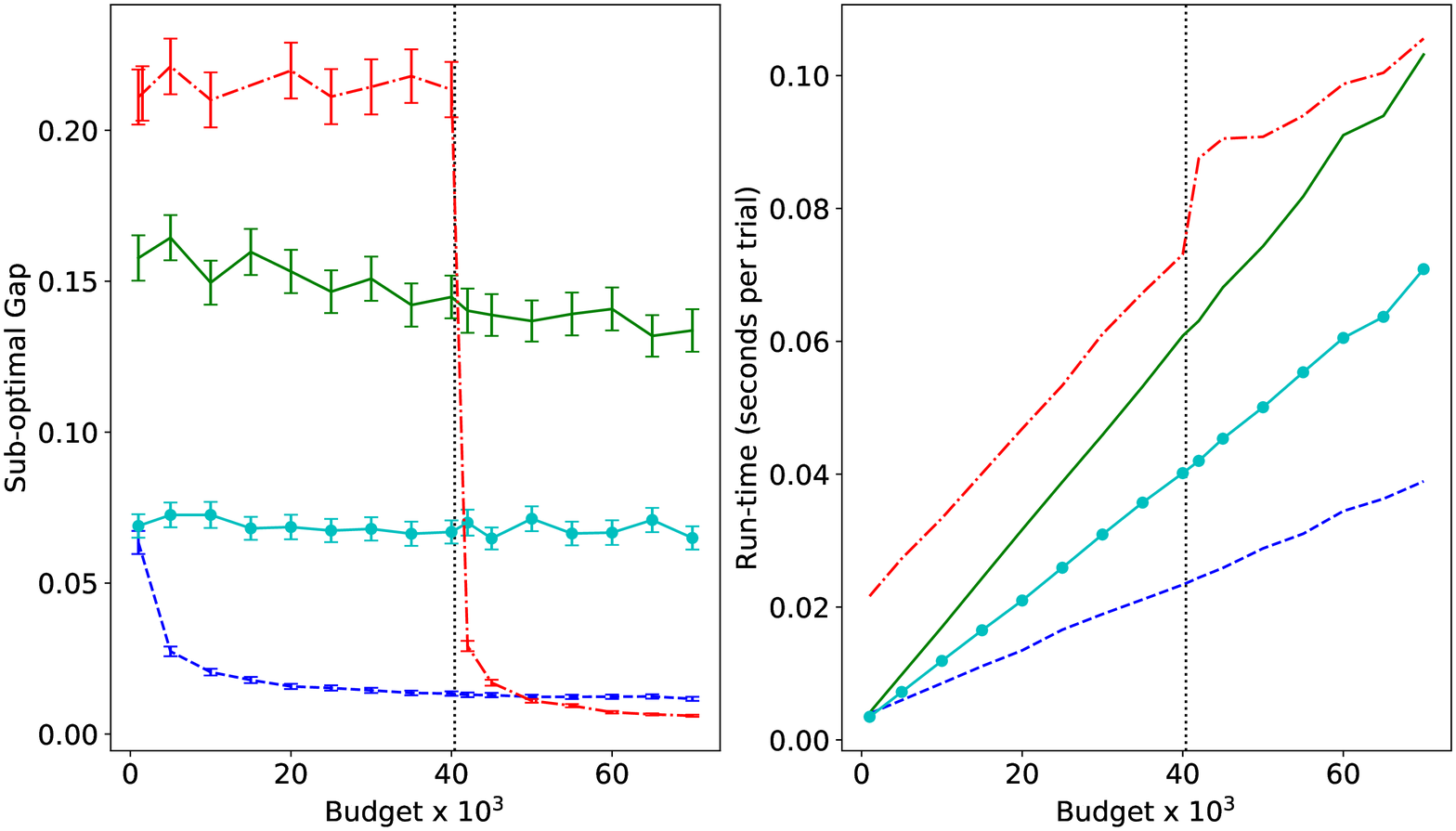}
    \caption{\label{fig:synthetic} Comparing the performance of different algorithms on synthetic data. Two left figures: Small synthetic graph $D=10$ (441 nodes). Two right figures: Large synthetic graph $D=100$ (40401 nodes). Figures show average sub-optimality gaps over 1000 trials and run-time per trial of the algorithms. Same legend (on the second figure) for all figures.}
\end{center}
\end{figure*}

\subsection{Applications in Graph-Structured Best-Arm Identification: Web Document Reranking}

To demonstrate the performance of our algorithm on real-world non-concave problems, we used data from Yandex Personalized Web Search Challenge to build a graph as follows. Each query forms a graph, whose nodes correspond to lists of 5 items (documents) for the given query. Two nodes are connected if they have 4 or more items in common at the same positions. The value of a node is a Bernoulli random variable with mean equal to the click-through rate of the associated list. The goal is to find the node with maximum click-through rate, i.e. the most relevant list. We chose the query that generated the largest possible graph (query no. 8107157) of 4514 nodes. As there were many small partitions in this graph, we took the largest partition as the input for our experiment. 
The resulting graph has 3992 nodes with degree varying from 1 to 171 (mean=35) and a maximum function value at $0.747$. 

\begin{figure}
\begin{center}
    \includegraphics[width=\columnwidth]{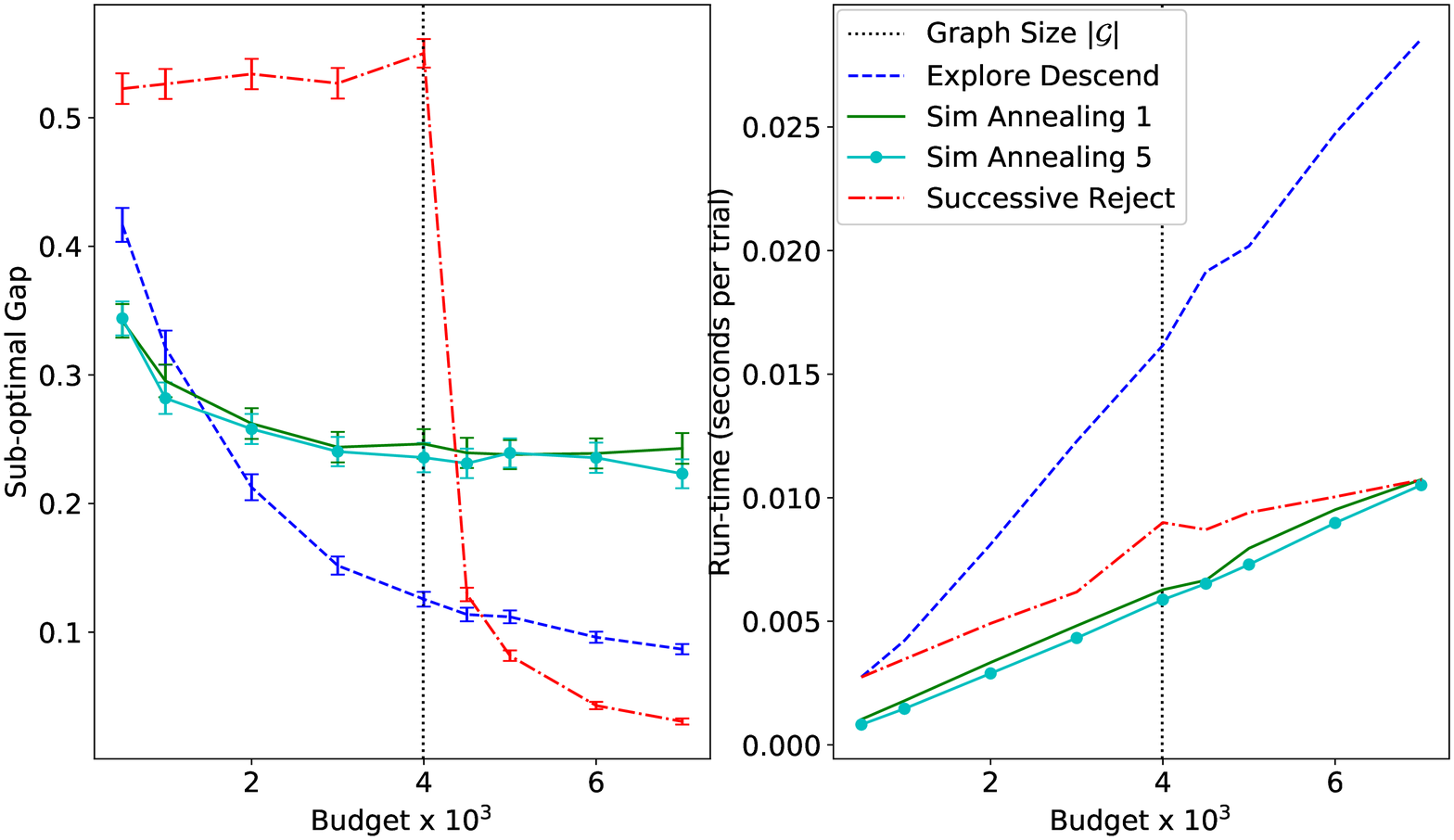}
    \caption{\label{fig:abtest} Comparing the performance of algorithms on web document reranking data. Left: Average sub-optimal gap over 1000 trials. Right: Run-time (s) per trial. Same legend for both figures.}
\end{center}
\end{figure}

For non-concave graphs, \textsc{Explore-Descend} needs to make multiple restarts. We set the number of restarts to $1+budget/1000$ and allocate the $budget$ equally between all restarts, then, for each restart, equally between each node in the path. All other parameters are the same as before. 

The results of the experiment is presented in Figure~\ref{fig:abtest}. In the intended setting, our algorithms significantly outperform \textsc{Successive Reject}. For very small budget, \textsc{Simulated Annealing} is better than \textsc{Explore-Descend}, but this is reversed as the budget gets bigger. Additionally, for this graph we don't see the big advantage of \emph{Sim Annealing 5} over \emph{Sim Annealing 1} as was the case before. Outside of the intended setting, \textsc{Successive Reject} quickly becomes the best algorithm when the budget gets larger than the graph size. Although, this algorithm requires global information about the graph, which may not be always feasible.

\subsection{Applications in Graph-Based Nearest-Neighbor Classification}
\label{sec:gnns}

\begin{figure*}[h]
    \centering
    \subfloat[]{{\includegraphics[width=0.25\textwidth]{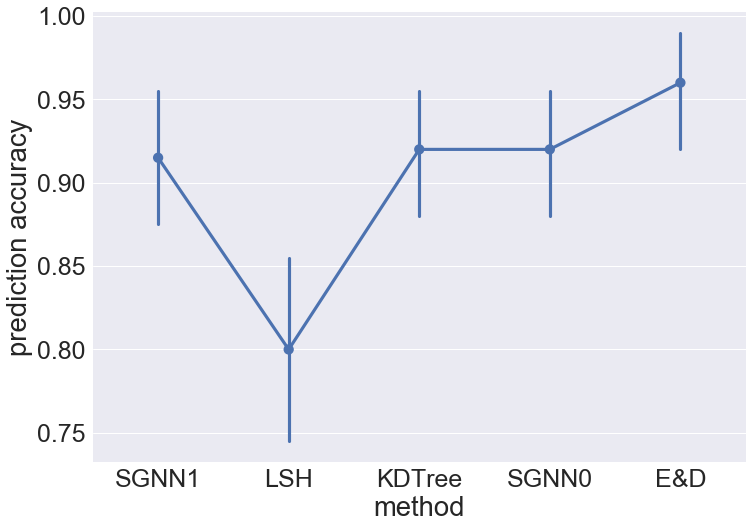} 	}}
	\subfloat[]{{\includegraphics[width=0.25\textwidth]{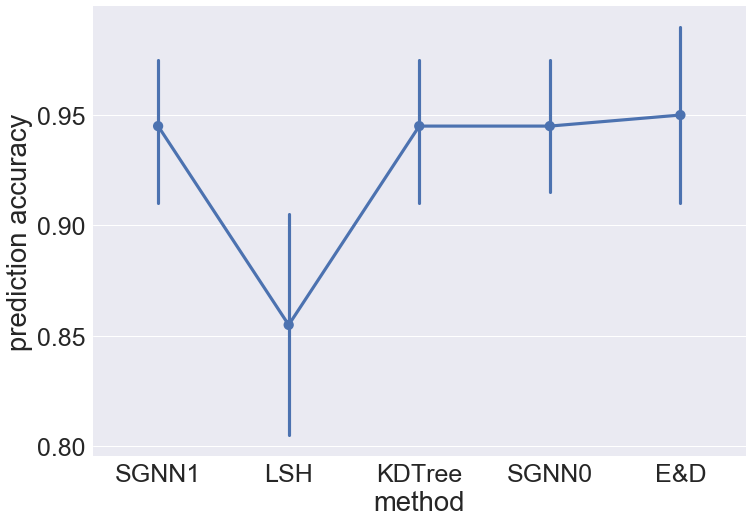} 	}} 	 
    \subfloat[]{{\includegraphics[width=0.25\textwidth]{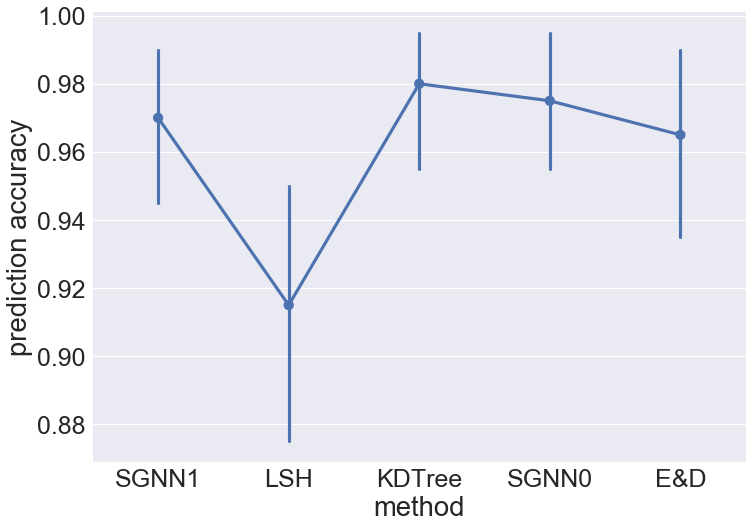} 	}} 
   \subfloat[]{{\includegraphics[width=0.25\textwidth]{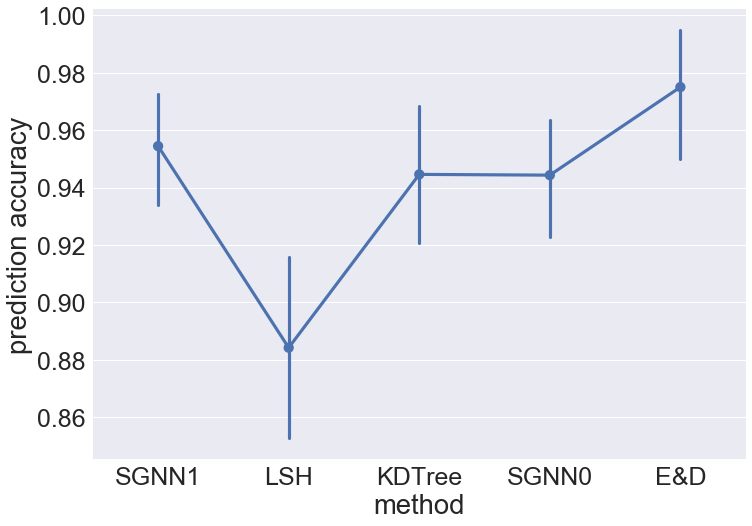} 	}} \\
    \subfloat[]{{\includegraphics[width=0.25\textwidth]{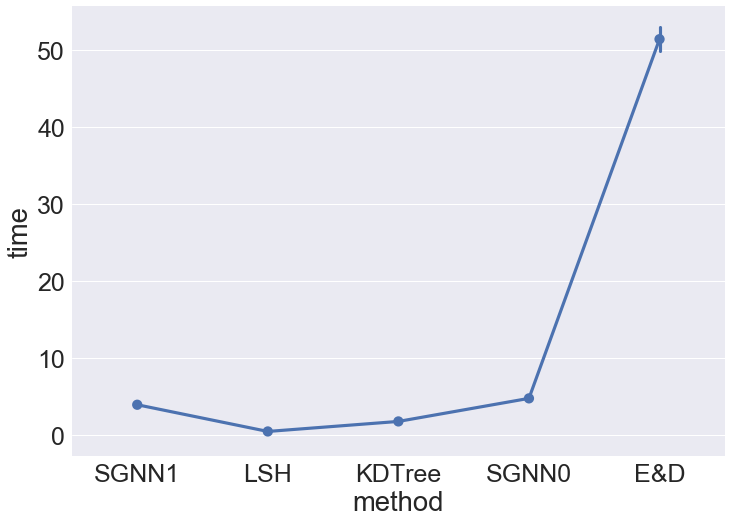} 	}}
	\subfloat[]{{\includegraphics[width=0.25\textwidth]{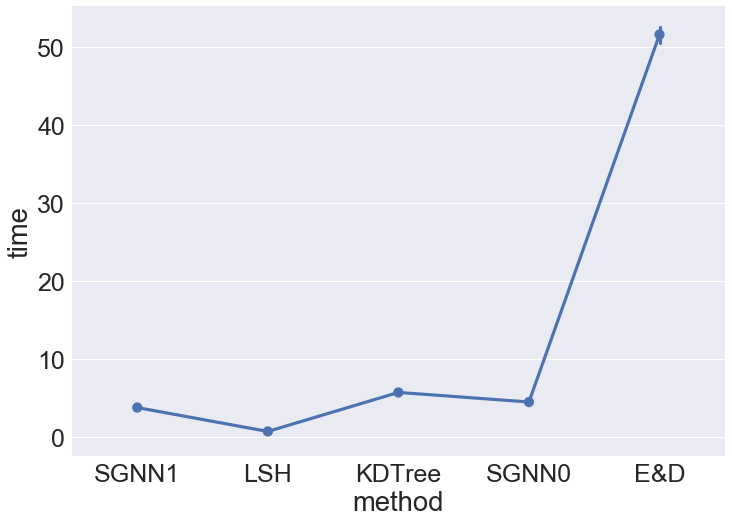} 	}} 	 
    \subfloat[]{{\includegraphics[width=0.25\textwidth]{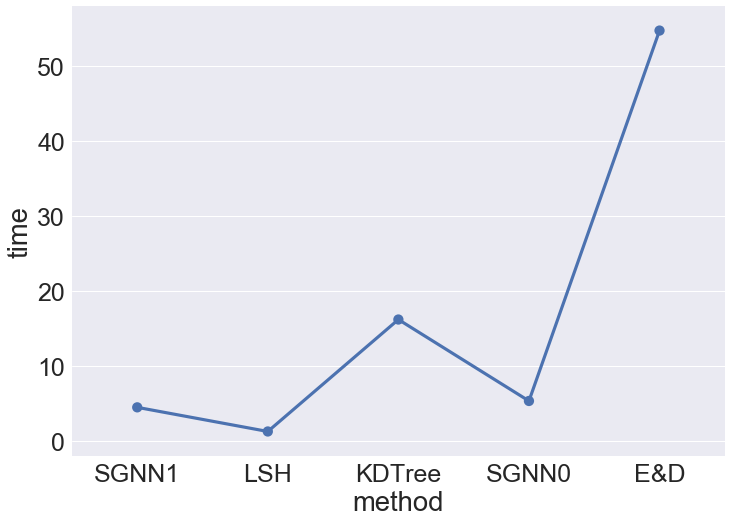} 	}} 
   \subfloat[]{{\includegraphics[width=0.25\textwidth]{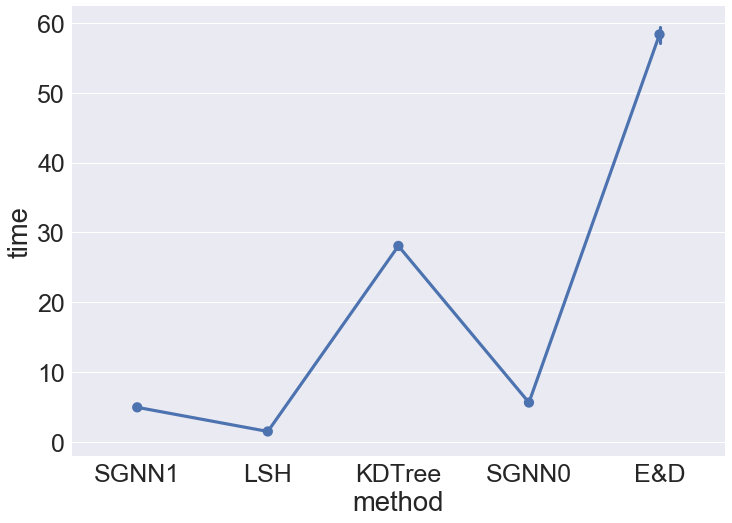} 	}}    
    \vspace{0cm}
    \caption{Prediction accuracy and running time of different methods on MNIST dataset as the size of training set increases. (a,e) Using 25\% of training data. (b,f) Using 50\% of training data. (c,g) Using 75\% of training data. (d,h) Using 100\% of training data.} \label{fig:mnist}%
\end{figure*}

In this section, we use the proposed graph-based optimization methods in a graph-based nearest neighbor search problem. 
The graph-based nearest neighbor method takes a training set, constructs a proximity graph on this set, and when queried in the test phase, performs a hill-climbing search to find an approximate nearest neighbor. More details are given in Appendix~\ref{app:gnns}. This procedure is particularly well suited to big-data problems; In such problems, points will have close neighbors, and so the geodesic distance with respect to even a simple metric such as Euclidean metric should provide an appropriate metric. Further, and as we will show, computational complexity of the graph-based method in the test phase scales gracefully with the size of the training set, a property that is particularly important in big data applications. The intuition is that, in big-data problems, although a descend path might be longer,  the objective function is generally more smooth and hence easier to minimize. 

We apply the local search and simulated annealing methods along with an additional smoothing to the problem of nearest neighbor search. For \textsc{Simulated Annealing}, we call the resulting algorithm SGNN for Smoothed Graph-based Nearest Neighbor search. The pseudocode of the algorithm and more details are in Appendix~\ref{app:gnns}. The \text{Explore-Descend} is denoted by E\&D in these experiments. 
We compared the proposed methods with the state-of-the-art nearest neighbor search methods (KDTree and LSH) in two image classification problems (MNIST and COIL-100). 
In an approximate nearest neighbor search problem, it is crucial to have sublinear time complexity, and thus \textsc{SpectralBandit} and \textsc{SuccessiveReject} are not applicable here. 

Figure~\ref{fig:mnist} (a-d) shows the accuracy of different methods on different portions of MNIST dataset. The graphs in these experiments are not concave, but $(\alpha=0.3, c=0, r=2)$--nearly concave by Definition~\ref{defn:nearlyconvex}. The results for COIL-100 are shown in Appendix~\ref{app:gnns}. 
As the size of training set increases, the prediction accuracy of all methods improve. Figure~\ref{fig:mnist} (e-h) shows that the test phase runtime of the SGNN method has a more modest growth for larger datasets. In contrast, KDTree becomes much slower for larger training datasets. The LSH method generally performs poorly, and it is hard to make it competitive with other methods. 
When using all training data, the SGNN method has roughly the same accuracy, but it has less than 20\% of the test phase runtime of KDTree. 


\section{Conclusions and Future Work}

 
We studied sample complexity of stochastic optimization of graph functions. We defined a notion of convexity for graph functions, and we showed that under the convexity condition, a greedy algorithm and the simulated annealing enjoy sample complexity bounds that are independent of the size of the graph. An interesting future work is the study of cumulative regret in this setting. 

We showed effectiveness of the proposed techniques in a web document re-ranking problem as well as a graph-based nearest neighbor search problem. The computational complexity of the resulting nearest neighbor method scales gracefully with the size of the dataset, which is particularly appealing in big-data applications. Further quantification of this property remains for future work.  


\section*{Acknowledgement}

TN was supported by the Australian Research Council Centre of Excellence for Mathematical and Statistics Frontiers (ACEMS).

\newpage
\bibliography{main}

\newpage
\onecolumn
\appendix
\section{More Details for Experiments}
\label{app:gnns}

First, we explain the graph-based nearest neighbor search.   
Let $N$ be a positive integer. Let $\G$ be a proximity graph constructed on dataset $\V$ in an offline phase, i.e. $\V$ is the set of nodes of $\G$, and each point in $\V$ is connected to its $N$ nearest neighbors with respect to some distance metric $\ell:\Real^d\times\Real^d\rightarrow\Real$. In our experiments, we use the Euclidean metric. Given the graph $\G$ and the query point $y\in\Real^d$, the problem is reduced to minimizing function $f = \ell(.,y)$ over a graph $\G$. The algorithm is shown in Figures~\ref{alg:opt-rr} and~\ref{alg:simulatedannealing}. The SGNN continues for a fixed number of iterations. In our experiments, we run the simulated annealing procedure for $\log n$ rounds, where $n$ is the size of the training set. See Figure~\ref{alg:simulatedannealing} for a pseudo-code. Finally, the SGNN runs the simulated annealing procedure several times and returns the best outcome of these runs. The resulting algorithm with random restarts is shown in Figure~\ref{alg:opt-rr}. The above algorithm returns an approximate nearest neighbor point. To find $K$ nearest neighbors for $K>1$, we simply return the best $K$ elements in the last line in Figures~\ref{alg:opt-rr}. 
We use $K=50$ approximate nearest neighbors to predict a class for each given query. We construct a directed graph $\G$ by connecting each node to its $N=30$ closest nodes in Euclidean distance. For smoothing, we tried random walks of length $T=1$ and $T=2$. This means that, to evaluate a node, we run a random walk of length $T$ from that node and return the observed value at the stopping point as an estimate of the value of the node. This operation smoothens the function, and generally improves the performance. The SGNN method with $T=1$ is denoted by SGNN(1), and SGNN with $T=0$, i.e. pure simulated annealing on the graph, is denoted by SGNN(0). For the SGNN algorithm, the number of rounds is $J=\log(\text{training size})$ in each restart. 

\begin{figure}
\begin{center}
\framebox{\parbox{12cm}{
\begin{algorithmic}
\STATE \textbf{Input: } Number of random restarts $I$, number of hill-climbing steps $J$, length of random walks $T$.
\STATE Initialize set $U=\{\}$
\FOR{$i:=1,\dots,I$}
\STATE Initialize $x = $ random point in $S$. 
\STATE $x' = \textsc{Smoothed-Simulated-Annealing}(x, T, J)$
\STATE $U=U\cup \{x'\}$
\ENDFOR
\STATE Return the best element in $U$
\end{algorithmic}
}}
\end{center}
\caption{The Optimization Method with Random Restarts}
\label{alg:opt-rr}
\end{figure}

\begin{figure}
\begin{center}
\framebox{\parbox{12cm}{
\begin{algorithmic}
\STATE \textbf{Input: } Starting point $x\in S$, number of hill-climbing steps $J$, length of random walks $T$. 
\FOR{$j:=1,\dots,J$}
\STATE Perform a random walk of length $T$ from $x$. Let $y$ be the stopping state.
\STATE Let $\widehat f(x, s) = f(y)$ 
\STATE Let $u$ be a neighbor of $x$ chosen uniformly at random.
\STATE Perform a random walk of length $T$ from $u$. Let $v$ be the stopping state.
\STATE Let $\widehat f(u, s) = f(v)$.
\IF{$f(v)\le f(y)$}
\STATE Update $x = u$  
\ELSE
\STATE Temperature $\tau=1-j/J$
\STATE With probability $e^{(f(y)-f(v))/\tau}$, update $x=u$
\ENDIF
\ENDFOR
\STATE Return $x$
\end{algorithmic}
}}
\end{center}
\caption{The \textsc{Smoothed-Simulated-Annealing} Subroutine}
\label{alg:simulatedannealing}
\end{figure}

The graph based nearest neighbor search has been studied by~\cite{Arya-Mount-1993, Brito-1997, Eppstein-1997, Miller-Teng-1997, Plaku-Kavraki-2007, Chen-Fang-Saad-2010, Connor-Kumar-2010, Dong-Charikar-Li-2011, Hajebi-2011, Wang-Wang-2012}. 
In the worst case, construction of the proximity graph has complexity $O(n^2)$, but this is an offline operation. 
Choice of $N$ impacts the prediction accuracy and computation complexity; smaller $N$ means lighter training phase computation, and heavier test phase computation (as we need more random restarts to achieve a certain prediction accuracy). Having a very large $N$ will also make the test phase computation heavy.

We used the MNIST and COIL-100 datasets, that are standard datasets for image classification. The MNIST dataset is a black and white image dataset, consisting of 60000 training images and 10000 test images in 10 classes. Each image is $28\times 28$ pixels. The COIL-100 dataset is a colored image dataset, consisting of 100 objects, and 72 images of each object at every 5x angle. Each image is $128\times 128$ pixels, We use 80\% of images for training and 20\% of images for testing. 

For LSH and KDTree algorithms, we use the implemented methods in the scikit- learn library with the following parameters. For LSH, we use 
LSHForest with min hash match=4, \#candidates=50, \#estimators=50, \#neighbors=50, radius=1.0, radius cutoff ratio=0.9. For KDTree, we use leaf size=1 and $K$=50, meaning that indices of 50 closest neighbors are returned. The KDTree method always significantly outperforms LSH. For SGNN, we pick the number of restarts so that all methods have similar prediction accuracy.

Figure~\ref{fig:coil} (a-d) shows the accuracy of different methods on different portions of COIL-100 dataset. As the size of training set increases, the prediction accuracy of all methods improve. Figure~\ref{fig:coil} (e-h) shows that the test phase runtime of the SGNN method has a more modest growth for larger datasets. In contrast, KDTree becomes much slower for larger training datasets. When using all training data, the proposed method has roughly the same accuracy, while having less than 50\% of the test phase runtime of KDTree. 
Using the exact nearest neighbor search, we get the following prediction accuracy results (the error bands are 95\% bootstrapped confidence intervals): with full data, accuracy is $0.955\pm 0.01$; with 3/4 of data, accuracy is $0.951\pm 0.01$; with 1/2 of data, accuracy is $0.943\pm 0.01$; and with 1/4 of data, accuracy is $0.932\pm 0.01$.

\begin{figure*}
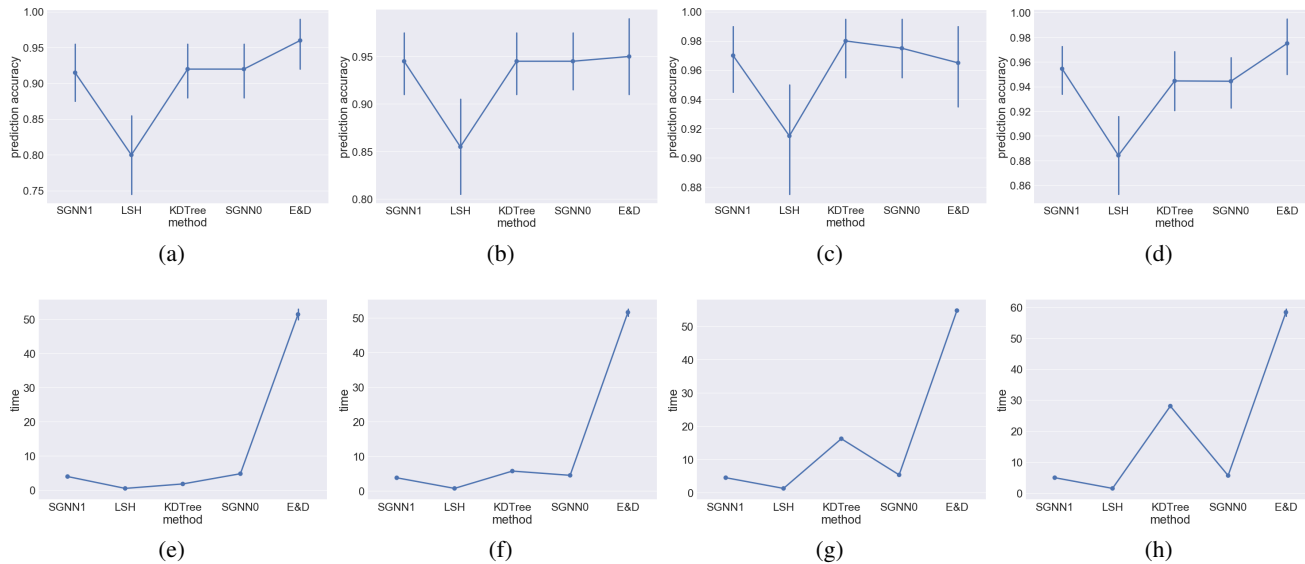

    \centering
    \subfloat[]{{\includegraphics[width=0.25\textwidth]{./Plots/25PercentAccuracy.png} 	}}
	\subfloat[]{{\includegraphics[width=0.25\textwidth]{./Plots/50PercentAccuracy.png} 	}} 	 
    \subfloat[]{{\includegraphics[width=0.25\textwidth]{./Plots/75PercentAccuracy.png} 	}} 
   \subfloat[]{{\includegraphics[width=0.25\textwidth]{./Plots/100PercentAccuracy.png} 	}} \\
    \subfloat[]{{\includegraphics[width=0.25\textwidth]{./Plots/25PercentTime.png} 	}}
	\subfloat[]{{\includegraphics[width=0.25\textwidth]{./Plots/50PercentTime.png} 	}} 	 
    \subfloat[]{{\includegraphics[width=0.25\textwidth]{./Plots/75PercentTime.png} 	}} 
   \subfloat[]{{\includegraphics[width=0.25\textwidth]{./Plots/100PercentTime.png} 	}} 
    \vspace{0cm}
    \caption{Prediction accuracy and running time of different methods on COIL-100 dataset as the size of training set increases. (a,e) Using 25\% of training data. (b,f) Using 50\% of training data. (c,g) Using 75\% of training data. (d,h) Using 100\% of training data.} \label{fig:coil}%
\end{figure*}

Next, we study how the performance of SGNN changes with the length of random walks. We choose $T=2$ and compare different methods on the same datasets. The results are shown in Figure~\ref{fig:longwalk}. The SGNN(2) method outperforms the competitors. Interestingly, SGNN(2) also outperforms the exact nearest neighbor algorithm on the MNIST dataset. This result might appear counter-intuitive, but we explain the result as follows. Given that we use a simple metric (Euclidean distance), the exact $K$-nearest neighbors are not necessarily appropriate candidates for making a prediction; Although the exact nearest neighbor algorithm finds the global minima, the neighbors of the global minima on the graph might have large values. On the other hand, the SGNN(2) method finds points that have small values and also have neighbors with small values. This stability acts as an implicit regularization in the SGNN(2) algorithm, leading to an improved performance. 

\begin{figure}
    \centering
    \subfloat[]{{\includegraphics[width=0.25\textwidth]{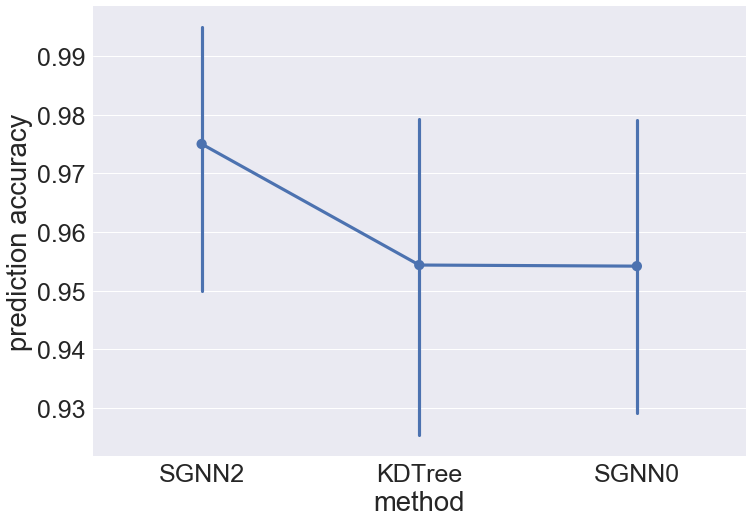} 	}}
	\subfloat[]{{\includegraphics[width=0.25\textwidth]{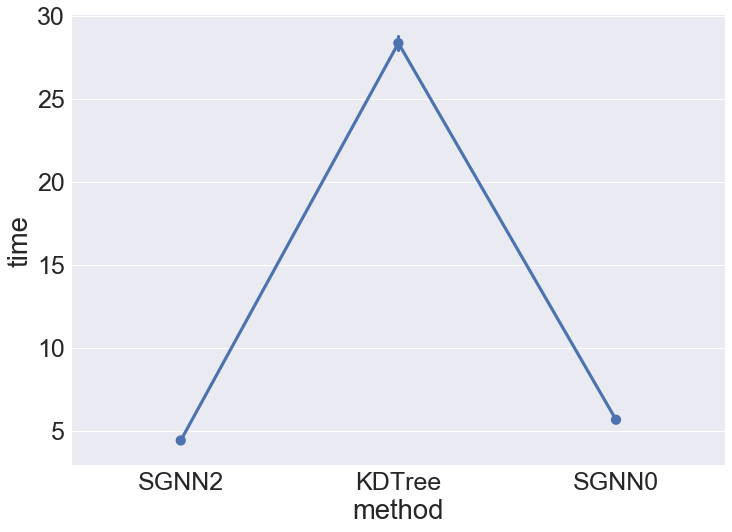} 	}} 	 
    \subfloat[]{{\includegraphics[width=0.25\textwidth]{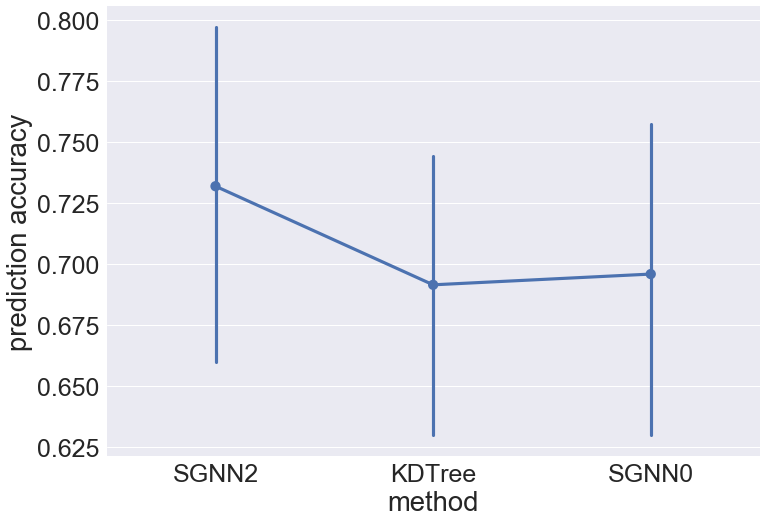} 	}} 
   \subfloat[]{{\includegraphics[width=0.25\textwidth]{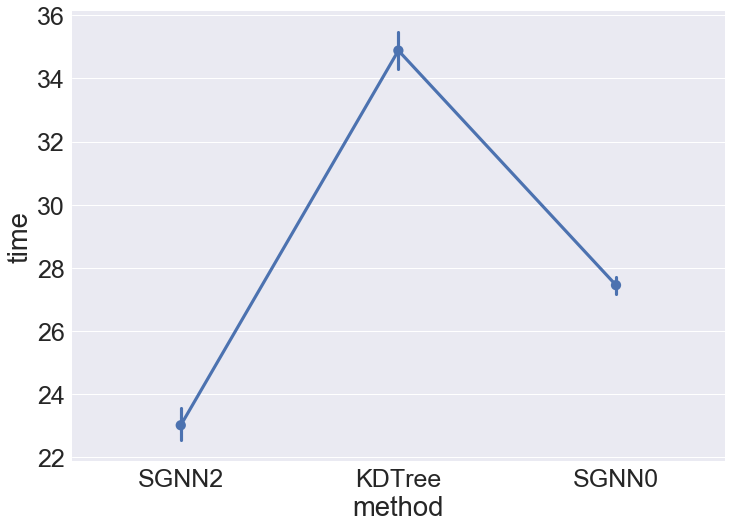} 	}} 
    \vspace{0cm}
    \caption{Prediction accuracy and running time of the SGNN method with random walks of length two (a) Accuracy on MNIST dataset using 100\% of training data. (b) Running time on MNIST dataset using 100\% of training data. (c) Accuracy on COIL-100 dataset using 100\% of training data. (d) Running time on COIL-100 dataset using 100\% of training data.} \label{fig:longwalk}%
\end{figure}

These results show the advantages of using graph-based nearest neighbor algorithms; as the size of training set increases, the proposed method is much faster than KDTree. 

\end{document}